\documentclass{article}
\pdfoutput=1

\usepackage[numbers]{natbib}
\usepackage[preprint]{neurips_2021}

\usepackage[utf8]{inputenc} 
\usepackage[T1]{fontenc}    
\usepackage{hyperref}       
\usepackage{url}            
\usepackage{booktabs}       
\usepackage{amsfonts}       
\usepackage{nicefrac}       
\usepackage{microtype}      
\usepackage{xcolor}         

\usepackage{bbm}
\usepackage{bm}
\usepackage{mathtools}
\usepackage{mathabx}

\usepackage{algorithmicx, algorithm}
\usepackage[noend]{algpseudocode}

\usepackage{amsthm}
\usepackage{graphicx}
\usepackage{subfiles}
\usepackage{array}

\algrenewcommand\algorithmicrequire{\textbf{input:}}
\algrenewcommand\algorithmicensure{\textbf{output:}}

\newtheorem{theorem}{Theorem}
\newtheorem*{theorem*}{Theorem}
\newtheorem{lemma}[theorem]{Lemma}

\theoremstyle{definition}
\newtheorem*{example}{Example}
\newtheorem{definition}[theorem]{Definition}

\theoremstyle{remark}

\newcommand{\z}{\zeta}
\newcommand{\C}{\mathbb{C}}
\newcommand{\R}{\mathbb{R}}

\newcommand{\F}{\mathcal{F}}
\newcommand{\iF}{\mathcal{F}^{-1}}

\newcommand{\one}{\mathbbm{1}}
\newcommand{\paren}[1]{\left\lparen{#1}\right\rparen}
\newcommand{\norm}[1]{\ensuremath{\left\|{#1}\right\|}}
\newcommand{\abs}[1]{\ensuremath{\left|{#1}\right|}}
\newcommand{\set}[1]{\ensuremath{\left\{{#1}\right\}}}
\newcommand{\Mag}[1]{\ensuremath{\operatorname{Mag}{\!\paren{#1}}}}
\newcommand{\Vol}[1]{\ensuremath{\operatorname{Vol}{\!\paren{#1}}}}

\mathchardef\mhyphen="2D 

\newcommand\restr[2]{{
  \left.\kern-\nulldelimiterspace 
  #1 
  \vphantom{\big|} 
  \right|_{#2} 
  }}

\title{Weighting vectors for machine learning: numerical harmonic analysis applied to boundary detection}

\author{%
  Eric Bunch\\
  American Family Insurance \\
  Madison, WI 53783\\
  \texttt{ebunch@amfam.com} \\
  \And
  Jeffery Kline \\
  American Family Insurance \\
  Madison, WI 53783\\
  \texttt{jklin1@amfam.com} \\
  \And
  
  Daniel Dickinson \\
  American Family Insurance \\
  Madison, WI 53783\\
  \texttt{ddickins@amfam.com}
  \And

  Suhaas Bhat \\
  American Family Insurance \\
  Madison, WI 53783\\
  \texttt{sbhat@amfam.com}
  \And
  Glenn Fung \\
  American Family Insurance \\
  Madison, WI 53783\\
  \texttt{gfung@amfam.com} \\
}

\begin{document}

\maketitle

\begin{abstract}
Metric space magnitude, an active field of research in algebraic topology, is a scalar quantity that summarizes the effective number of distinct points that live in a general metric space. The {\em weighting vector} is a closely-related concept that captures, in a nontrivial way, much of the underlying geometry of the original metric space.  Recent work has demonstrated that when the metric space is Euclidean, the weighting vector serves as an effective tool for boundary detection. We recast this result and show the weighting vector may be viewed as a solution to a kernelized SVM.  As one consequence, we apply this new insight to the task of outlier detection, and we demonstrate performance that is competitive or exceeds performance of state-of-the-art techniques on benchmark data sets.   Under mild assumptions, we show the weighting vector, which has computational cost of matrix inversion, can be efficiently approximated in linear time. We show how nearest neighbor methods can approximate solutions to the minimization problems defined by SVMs.
\end{abstract}

\section{Introduction}
\label{sec:introduction}
\begin{figure}[ht]
{
    \centering
    \begin{tabular}{cc}
        \includegraphics[width=0.29\textwidth,bb=0 0 700 700]{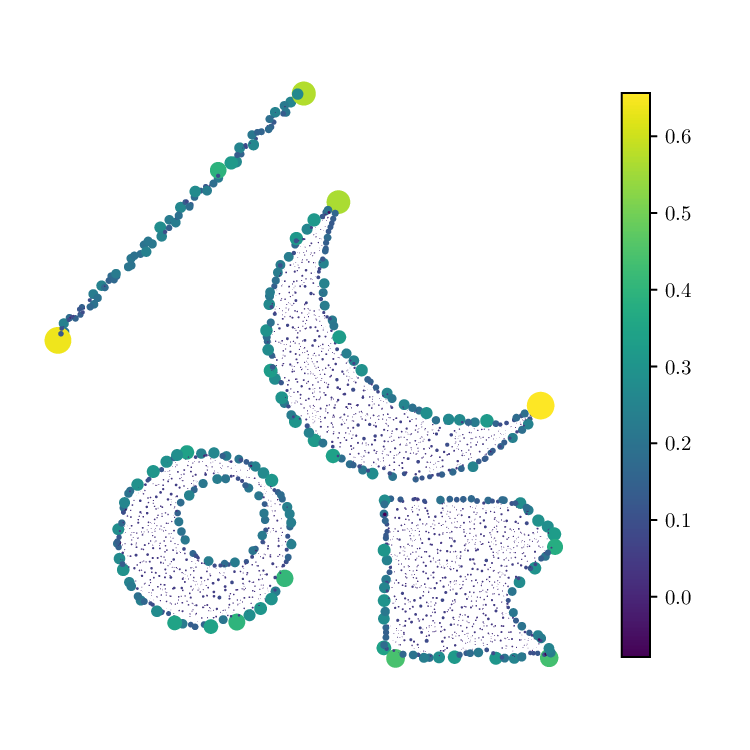}
        &
        \includegraphics[width=0.59\textwidth,bb=0 0 900 450]{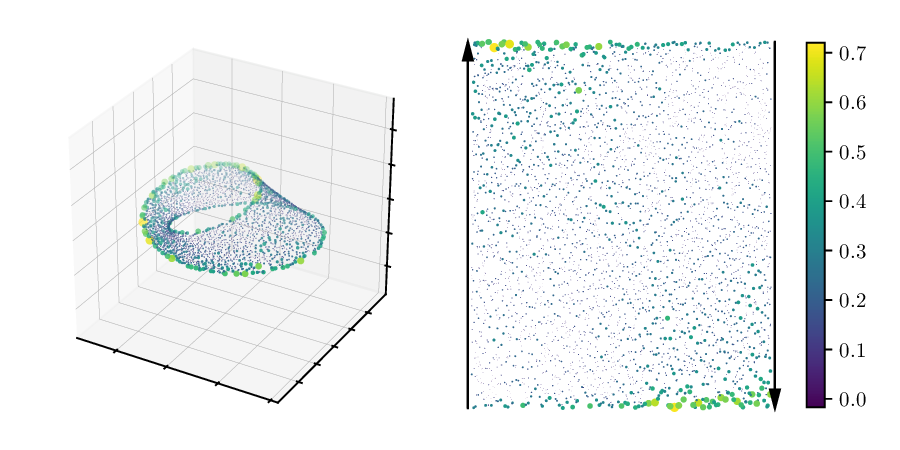}\hspace{2em} \\
    \end{tabular}
    \caption{A visualization of two weighting vectors.  The set in the left-hand figure is supported within four disjoint components, and they live in $\R^2$. The set in the right-hand figure is supported on an embedding of M\"{o}bius strip, and it lives in $\R^3$. In both images, the weight of each point is represented using color and point size.}
    \label{fig:numerical_example}
}
\end{figure}

{Magnitude} is a scalar quantity that has meaning for many different kinds of data, and as with other scalar quantities such as rank, diameter, and measure, it has wide applicability, an intuitive interpretation and a solid theoretical foundation. Magnitude has been discovered, and rediscovered multiple times in both practical and theoretical contexts. In this paper, our goal is to apply recent developments drawn from magnitude theory to machine learning, and to empirically demonstrate characteristics of magnitude that, while implicitly described by abstract theoretical results, have not, to our knowledge, been explicitly stated before, nor have they been leveraged for practical purpose.

Informally, magnitude aims to quantify the effective number of points in a space. Our aim is more subtle: we wish to identify \textit{which points} are considered ``effective'' and ``important.''  We do this using the {\em weighting vector}. The weighting vector appears naturally in the definition of magnitude, and we find that the weighting vector, under appropriate conditions, serves as an effective \textit{boundary detector}.   It is this behavior that makes the weighting vector especially well suited for machine learning tasks. 

\subsection{Background, notation and examples}
We now define magnitude and the weighting vector, we present examples, and we state several central theorems of the field. While our focus is largely on subsets of $\R^n$, we note that the concept {\em magnitude} and {\em weighting vector} can be defined for far more general types of sets.
\begin{definition}
Let $X$ be a finite metric space with metric $d$. Denote the number of points in $X$ by $\abs{X}$. The \textit{similarity matrix} of $X$ is defined to be $\z_X(i, j) := \exp(-d(x_i, x_j))$ for $1 \leq i,  j \leq \abs{X}$. Whenever the inverse of $\z_X$ exists, we define the \textit{weighting vector} of $X$ to be 
    \begin{align*}
        w_X \coloneqq \z_X^{-1}\one,
    \end{align*}
where $\one$ is the $\abs{X} \times 1$ column vector of all ones. The \textit{magnitude} of $X$ is defined to be the quantity
    \begin{align*}
        \Mag{X} \coloneqq \one^Tw_X = \one^T \z_X^{-1} \one.
    \end{align*}
That is, $\Mag{X}$ is the sum of all the entries of the weighting vector $w_X$.
\end{definition}

\begin{example}
When $X$ is a finite subset of Euclidean space, $\z_X$ is a symmetric positive definite matrix [Theorem 2.5.3, \cite{Leinster2013TheMO}]. In particular, $\z_X^{-1}$ is guaranteed to exist. Hence, the weighting vector and magnitude exist for finite subsets of $\R^n$.
\end{example}

\begin{example}
Given an undirected, unweighted graph $G$, one can define a metric space whose points are given by the vertices of $G$, and whose metric is taken to be the length of the shortest path between two vertices. The weighting vector of this metric space is not guaranteed to exist.
\end{example}

\begin{definition}\label{def:infinitesets}
For an arbitrary subset $X \subseteq \R^n$, the magnitude of $X$ is defined as
\begin{align*}
    \Mag{X} = \sup\{ \Mag{Y} \mid Y \text{ is a finite subset of } X \}.
\end{align*}
\end{definition}
\begin{example}
In 1 dimension, and for $t>0$, one has that $\Mag{[0,t]}=1+t/2$. This was shown by Leinster in  \cite{Leinster2013TheMO}. The magnitude of the ball with radius $r$ in $\R^{2n+1}$ is a rational function of $r$; this was recently demonstrated by Barcel\'{o} and Carbery~\cite{Barcelo18CptEuclid}.
\end{example}

For a finite metric space $(X, d)$, and any $t \in [0, \infty]$, we can define a new metric space $(tX, td)$ in the following way. The points of $tX$ are the same as those of $X$, and the metric $td$ is $d$ scaled by $t$: $td(x, y) \coloneqq t\cdot d(x, y)$. The \textit{magnitude function of $X$} is the map $t \mapsto \Mag{tX}$, and it is well-defined whenever $\z_{tX}$ is invertible. Although the inverse of $\z_{tX}$ may not be defined in general, it has been shown in [Proposition 2.2.6 \cite{Leinster2013TheMO}] that for finite subsets of $\R^n$, the magnitude function is analytic on $(0, \infty)$. We also have the following:
\begin{theorem}[Proposition 2.2.6 \cite{Leinster2013TheMO}]\label{theorem_t_inf_finite_set}
For $X \subset \R^n$ finite, $\lim_{t \rightarrow \infty} \Mag{tX} = \abs{X}$.
\end{theorem}

The above proposition is one of the reasons underlying the informal interpretation of magnitude as quantifying the effective number of points in a space. The following very recent theorem gives a connection between the magnitude of $X \subset \R^n$ and the $n$-volume of $X$.
\begin{theorem}[Theorem 1 \cite{Barcelo18CptEuclid}]\label{theorem_vol}
For $X \subset \R^n$ nonempty and compact, we have
\begin{align*}
    \lim_{t \rightarrow 0^+} \Mag{tX} 
    &= 
    1, \text{\hspace{3pt} and\hspace{3pt}}
    \lim_{t \rightarrow \infty}\frac{\Mag{tX}}{t^n} 
    = 
    \frac{\Vol{X}}{n!\Vol{B_n}},
\end{align*}
\noindent where $B_n \subset \R^n$ is the unit ball.
\end{theorem}

\subsection{Properties of the weighting vector}
The weighting vector plays a central role in the applications that are discussed below, but it is not clear by inspection of its definition that the individual entries of the weighting vector carry useful information.   To provide intuition about the weighting vector, we now highlight some key features.


Let $X\subset \R^n$ and $w_X$ its weighting vector. The entries of $w_X$ may be indexed in a canonical way by $x\in X$. We call $w_X(x)$, the weight of $x$. 
Since the similarity matrix $\zeta_{X}$ is positive definite, its inverse exists and is also positive definite. Thus, $\one'\zeta_{X}^{-1}\one = \one'w_{X}>0$. Although the average value of the entries of $w_X$ is guaranteed to be positive, it may happen that $w_X(x)<0$ holds for some $x\in X$.  

An heuristic argument led us to conjecture that the weighting vector ought to be useful as a device for boundary detection. Very recently, this heuristic was made rigorous with a formal theorem and proof \cite{meckesprivatecomm}. Further discussion about this development is in Section \ref{sec:boundarydetect}. 

It is not at all clear by inspection of the definition that a weighting vector might carry useful information about a set's boundary. There is early work empirically investigating the weighting vector for examples of finite metric spaces \cite{willerton2009heuristic}. The theoretical foundation of using weighting vectors for boundary detection   leverages the theory of Bessel potential functions and other machinery of harmonic analysis. 


\subsection{Related work, paper structure}

An early reference to the concept of magnitude occurs in~\cite{Solow1994}, where it was introduced as a way to measure biological diversity. However, the mathematical motivations were not divulged in this paper. Two decades later, Leinster placed the magnitude of a metric space within a formal mathematical framework using category theory~\cite{Leinster2013TheMO}.  This highly abstract perspective lead to the current era, where it is being explored through many different lenses, including functional analysis \cite{MeckesPosDef, Barcelo18CptEuclid}, harmonic analysis~\cite{meckes2015magnitude} and homology theory \cite{leinster2017magnitude}, where it has been shown to be equivalent to an Euler characteristic.   Much of the prior emphasis has been on a set's  magnitude, and this focus has overshadowed the potential utility of the weighting vector.

Recently, {topological data analysis} has emerged as an approach to the problem of describing the shape of high-dimensional data~\cite{Edelsbrunner2000TopologicalPA, Scopigno04persistencebarcodes, computingPH}. One particularly popular topic within this field is persistent homology \cite{Edelsbrunner2000TopologicalPA}.  Recent efforts have realized magnitude as the Euler characteristic of a homology theory, called magnitude homology \cite{leinster2017magnitude}. It has also been shown that there is a direct relationship between magnitude homology and persistent homology \cite{ otter2018magnitude}; however, the current paper is the first known application of magnitude directly to machine learning.


We now describe the remaining sections of this paper. Section~\ref{sec:boundarydetect} presents theoretical and heuristic justification as to why the weighting vector is a boundary detector. Section \ref{sec:weight_and_svms} details how the weighting vector can be seen as the solution to a generalized SVM. Section \ref{sec:nn_approx} presents methods to approximate the weighting vector with nearest neighbor methods. 
Section \ref{sec:outlier_detection} presents results of experiments run with a weight-based anomaly detection algorithm. We end with concluding remarks in Section \ref{sec:conclusions}.

\section{Boundary detection}
\label{sec:boundarydetect}
The purpose of this section is to state the theorem in~\cite{meckesprivatecomm}, which contributes rigor to the above discussion about boundary detection. We begin by offering informal comments about how the finite, discrete sets of the applications  relate to the infinite, continuous objects of the theorem. The background required for this initial part of the discussion is limited to basic familiarity with the Fourier transform. These comments also serve to present the moral case that weighting vectors ought to be useful as boundary detectors.  The background required to interpret the theorem's statement includes substantial familiarity with tempered distributions and related theory, which is beyond the scope of what can be presented here.


In the sequel, let $\F(f)$ denote the Fourier transform of a smooth function $f:\R^{n}\rightarrow \C$. Recall that under suitable conditions on $f$ one has that
\begin{align}
    (2\pi i\xi)^{k} \paren{\F{f}}(\xi) = \paren{\F\paren{\partial^{k}f}}(\xi).
    \label{eqn:Fdf}
\end{align}

Now suppose $X\subset[0,t]$ is a finite set of equispaced points selected from the interval $[0,t]$ (the equispacing condition may be relaxed to instead be a uniform random sample).  
The linear equation that defines the weighting vector is $\z_Xw = \one$.
This statement has, via Riemann summation, a continuous analogue expression for a ``weighting function'' $v$ for the entire interval, $[0,t]$. Let $h:=h_{1/2\pi}$. This analogue has the form,
\begin{align*}
      h*v(x) = \int_{\R} h(x-y)v(y)\,dy &= \int_{\R} e^{-\abs{x-y}}v(y)\,dy = I_{[0,t]}(x).
\end{align*}
By \cite{folland1999real} Chapter 8,
\begin{equation*}
    \paren{\F{I_{[0,t]}}}(\xi) 
    =
    \paren{\F{ (h*v)}}(\xi)
    =
    \paren{\F{h}}(\xi)\paren{\F{v}}(\xi)
    =
    \frac{2}{{1+4\pi^2\xi^2}} \paren{\F{v}}(\xi),
\end{equation*}
or 
\begin{align}
   \frac{1}{2}({1+4\pi^2\xi^2}) \F\paren{I_{[0,t]}}(\xi) = \paren{\F{v}}(\xi).
   \label{eqn:wdef}
\end{align}
Applying Eq.~\ref{eqn:Fdf} and the operator $\iF$ to Eq.~\ref{eqn:wdef}, one has
\begin{align}
    \frac{1}{2}\paren{ I_{[0,t]} - \partial^2 I_{[0,t]} }=v.
    \label{eqn:v}
\end{align}
In Eq.~\ref{eqn:v}, the term $\partial^2 I_{[0,t]}$ vanishes everywhere except at the points $0$ and $t$, where it behaves as second-order derivative operator.  Informally, $v$ is constant on the open interior $(0,t)$, and it approximates a discrete second-order derivative operator at the boundary points $0$ and $t$.  This informal argument may be adapted to $n>1$ dimensions, where $I_{[0,t]}$ is replaced by more general $A\subset\R^{n}$. 

We now turn to the theorem statement. The {\em Bessel potential space} is the Hilbert space of tempered distributions
\begin{align*}
    H^{s}\coloneqq \left\{ \phi\in S'(\R^{n}) : (1+\norm{\cdot}^2)^{s/2}\F\phi\in L^2(\R^n)\right\}
\end{align*}
that is equipped with norm
\begin{align*}
    \norm{\phi}_{H^{s}}\coloneqq 
    \paren{
    \int_{\R^{n}} \paren{ 1 + \norm{\xi}^{2}}^{s} 
    \abs{\paren{\F\phi}(\xi)}^2 \,d\xi
    }^{1/2}.
\end{align*}
For for compact $A\subset\R^{n}$, the definition of a {\em weighting}, as well as necessary and sufficient conditions for $A$ to have a weighting, can be found in Definition 3.3 and Theorem 4.1 of~\cite{meckes2015magnitude}.
\begin{theorem}(\cite{meckesprivatecomm})
Let $n$ be odd, $A\subset\R^{n}$ compact with weighting $z\in H^{-(n+1)/2}(\R^{n})$, and let $\lambda_A$ denote Lebesgue measure restricted to $\operatorname{int} A$.  Then 
\begin{align}z=\frac{1}{n!\omega_n}\lambda_A+\nu\label{eqn:z}\end{align} for some $\nu\in H^{-(n+1)/2}(\R^{n})$ that is supported on $\partial A$. The constant $\omega_n\coloneqq \pi^{n/2}/\Gamma(n/2+1)$ is the volume of the unit $n$-ball.
\end{theorem}

The decomposition of $z$ stated in Eq.~\ref{eqn:z} agrees with informally-derived expression for $v$ stated in Eq.~\ref{eqn:v}. Numerically, we find that $n$ odd does not seem to be required.  Finally, we note that under extra regularity assumptions, a similar result follows from Theorem 5 of~\cite{Barcelo18CptEuclid}.

\subsection{Interpreting Transformer Language Models with Weighting Vectors}

\begin{figure}[ht]
\centering
\includegraphics[width=\textwidth,bb=0 0 1000 200]{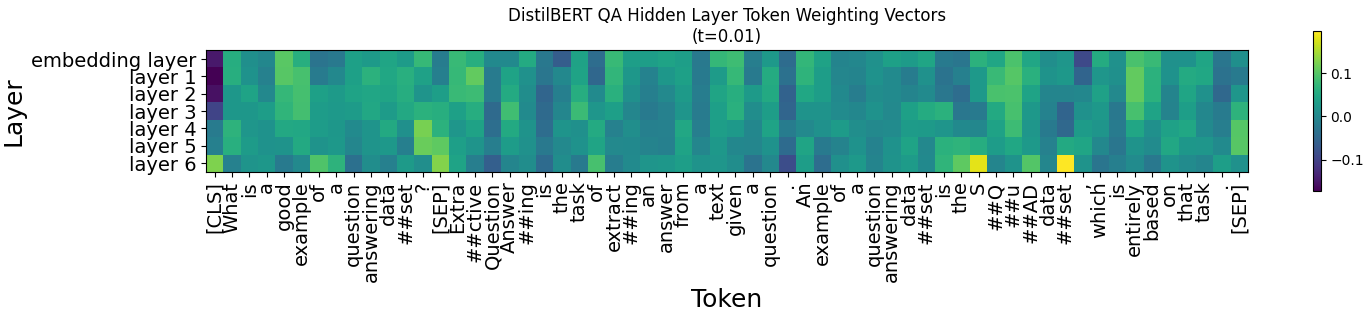}
\caption{The layers of a DistilBERT QA model manipulate the token-vector space so that the start and end tokens of the answer span are on a boundary.}
\label{fig:interpret_bert_QA}
\end{figure}

In the sequel, we use the theory from section \ref{sec:boundarydetect} to establish relationships between magnitude and well-known tools in the machine learning community and develop novel uses for the boundary-detection properties of the weighting vector. Before doing so, however, we begin with a straightforward application of the weighting vector: interpreting the internal workings of modern transformer-based language models such as BERT \cite{devlin-etal-2019-bert}. Using a DistilBERT \cite{sanh2020distilbert} pre-trained extractive question-answering model \cite{wolf-etal-2020-transformers}, we compute the weighting vector of the tokens in the 768-dimensional embedding space for each layer.

As a concrete example, shown in figure \ref{fig:interpret_bert_QA}, the question and context fed to the model are "What is a good example of a question answering dataset?", and "Extractive question answering is the task of extracting an answer from a text given a question. An example of a question answering dataset is the SQuAD dataset, which is based on that task.", respectively. The desired answer, of course is "SQuAD dataset". In the final embedding layer of the transformer, the tokens that correspond to the beginning and end of "SQuAD dataset" have the largest weighting vector components. Notably, these are the tokens that the model has been trained to return. This implies that under the hood of the transformer, the vector embeddings of the correct tokens are being pushed towards a boundary of a region within the embedding space, and the remaining tokens remain in the interior of that boundary. 

Our analysis shows that this behavior holds in general. Using 1000 random question-context pairs from the SQuAD v2 dataset \cite{2016arXiv160605250R} and a value of $t=0.01$, the start and end tokens in the last embedding layer on average each have weights that are in the $96^{th}$ percentile of token weights for the pairing. 

We also note that as we trace the progress of the [CLS] token through the different layers, it's weighting vector component goes from very negative to quite positive, and this transition only begins around layer 3. This hints that in the DistilBERT model, many layers are necessary in order to imbue the embeddings with a solid semantic understanding of the input tokens.

\section{Relation to generalized support vector machines} \label{sec:weight_and_svms}

\begin{figure}[htbp]
    \centering
    \includegraphics[width=0.9\textwidth,bb=0 0 1800 900]{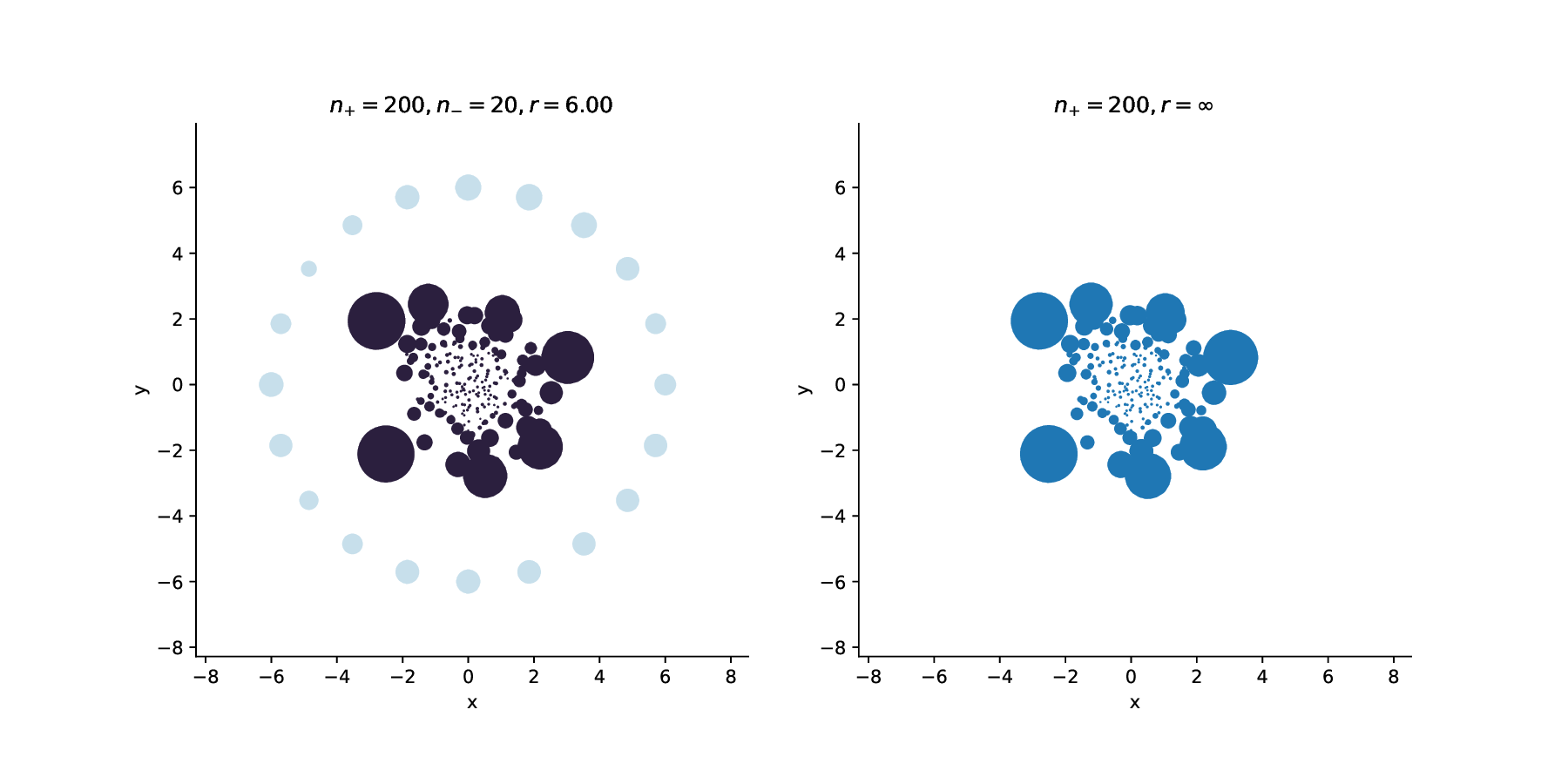}
    \caption{Classification of  discrete sets in the plane using a kernelized SVM with the Laplacian kernel,  $K(x',y)\coloneqq \exp\paren{-t\norm{x-y})}$.  Point size is based on the modulus of the SVM solution. One of the classes is equi-distributed on a circle of radius $r=6$ (left) and $r=\infty$ (right). The solution of the SVM used in the image at right is equivalent the solution of a one-class SVM, and since the kernel is the Laplacian, this solution also equals the weighting vector of the cluster.}
    \label{fig:2class}
\end{figure}

In this section, we recast the weighting vector as a solution to a one-class support vector machine. We begin with some background on support vector machines (SVM) \cite{svm, feat_selection_svm, learning_from_data, learning_theory_vapnik, generalizedSVM}.  The task of a linear support vector machine aims to classify $m$ points in $\R^n$ as belonging to one of two classes, either  $A^+$ or $A^-$. If we represent the data by the $m\times n$ matrix $A$, and the class membership as specified by a diagonal matrix $D$, with $+1$ or $-1$ along the diagonal, the linear support vector machine has the form
\begin{align*}
    \min_{w,\gamma,y} &\quad\nu \one'y + \norm{w}_1\\
    \textrm{subject to}&\quad D(Aw-\one \gamma) + y\geq \one\\
    &\quad y\geq 0,
\end{align*}
where $\norm{w}_1\coloneqq \sum_{i}\abs{w_i}$ denotes the 1-norm of $w$. The variable $y$ is a vector of slack variables, and the optimal $(w,\gamma)$ characterize a separating boundary between $A^+$ and $A^-$.  The decision boundary consists of $x\in\R^n$ that satisfies the linear equality $w'x=\gamma$. 

In~\cite{generalizedSVM}, it is shown that a very general {\em kernelized} support vector machine may be formulated.  The formulation established there, and which we use here, is more general than the common approach that employs the kernel-trick to specify a SVM's nonlinear decision boundary.  A {\em kernel function}, $K(A,B)$, is any map from $R^{m\times n}\times R^{n\times l}\rightarrow\R^{m\times l}$. If $x,y \in\R^{n}$ are column vectors, then $K(x',A')$ is a row vector in $R^{m}$, and $K(x',y)$ is a scalar.  In the established notation, the kernelized support vector machine program is
\begin{align*}
    \min_{u,\gamma,y} &\quad \one' y + s \norm{u}_1\\
    \textrm{subject to}&\quad D(K(A,A')Du-\one\gamma) + y\geq \one\\
    &\quad y\geq 0.
\end{align*}
If $(u,\gamma)$ is a solution to this program, then the  decision boundary consists of $x\in\R^{n}$ that satisfy $K(x',A')Du=\gamma$.

If this program has kernel function $K(x',y)\coloneqq \exp(-t\norm{x-y})$, regularization parameter $s=0$, and a classification task with only one class, then $K(A,A')=\zeta_A$ and the program reduces to
\begin{align}
\min_{u,\gamma} \norm{ \min ({\zeta_A u-(1+\gamma)\one},  0)}_1\label{eqn:ubzeta}
\end{align}
For any pair $(u,\gamma)$, one has that
\begin{align*}
     \norm{ \min ({\zeta_A u-(1+\gamma)\one, 0}}
\leq \norm{\zeta_A u-(1+\gamma)\one}_1.
\end{align*}
When $A$ is a finite subset of $\R^n$, the matrix inverse of $\zeta_A$ exists, and so the term on the right, which is an upper bound, vanishes when $u=(1+\gamma)\zeta_A^{-1}\one$, which is (up to normalization) equal to the weighting vector.  Since this is an upper bound, the term on the left, which is nonnegative, must also vanish, and this shows the weighting vector is an optimal solution to the generalized SVM program.

This allows us to reinterpret the weighting vector as a support vector of a generalized SVM, and it also aligns with the empirical and theoretical observations made elsewhere that the entries of the weighting vector can be used as an effective boundary detector.

This view also provides an explanation for the empirical observation that when $t>0$ is very small, the weighting vector assigns larger weights to the extremal points of the convex hull of a metric space.  Consider two labeled sets in an arrangement like that illustrated on the left in Figure~\ref{fig:2class}.  In the left image, one cluster of points is surrounded by a set of points distributed on a circle. This pair of sets was used to train a two-class SVM with kernel function $K(x',y)\coloneqq \exp\paren{-t\norm{x-y})}$. Letting $(w,\gamma)$ represent the solution to this SVM, the image uses $\abs{w}$ to represent  the point sizes in this image.  In the image on the left, the largest entries of the solution act as supporting vectors of the maximum margin decision boundary, and those points are nearest to the opposite class.  On the right, point size is determined by the weighting vector applied to the just the clustered class, which we have just seen, is a one-class SVM. The one-class SVM can formally be interpreted as a 2-class SVM, except that the points belonging to a second class surround the first image and live on a circle with infinite radius.

\section{Approximation by nearest neighbors}\label{sec:nn_approx}


In this section show that by using the work of Leinster on magnitude \cite{Leinster2013TheMO} we can give new insight to the solution of generalized SVMs induced by a kernel $K$. For certain kernels $K$, $w$ can be approximated by $1/f$, where $f$ is the kernel density estimator: $f(x) = \frac{1}{n}\sum_{i=1}^nK(x, x_i)$.

Boundary detection has shown up in the context of approximating viability kernels in viability theory \cite{viability_theory}. In \cite{viab_kernels_svms} it was shown that the solutions to kernelized SVMs, and specifically an SVM with a Gaussian kernel, perform well when approximating the boundary of a viability kernel. There is also work done in \cite{kdtree_vol} to use the k-d tree structure to find the boundary of a subset of $\R^d$ via oracle-labeled sample points, as well as to compute the volume of the enclosed region. After an exchange of the Gaussian kernel of \cite{viab_kernels_svms} for a Laplacian kernel, the result of this section will allow us to see the two approaches of \cite{viab_kernels_svms} and \cite{kdtree_vol} to be closely related.

\begin{definition} \label{def:scattered}
(\cite{Leinster2013TheMO}, Def. 2.1.2) A dataset $X$ is \emph{scattered} if ${\rm e}^{-\epsilon} < \frac{1}{n - 1}$ where $\epsilon$ is the smallest distance between distinct points in $X$.
\end{definition}

\begin{theorem} \label{thm:w_kd_error_bound}
For $X$ scattered, and $x \in X$,

\begin{equation} \label{eqn:w_kd_error_bound}
    \abs{w(x) - \frac{1}{nf(x)}} \leq \frac{n(n-1)^2{\rm e}^{-2\epsilon} + n(n-1){\rm e}^{-\epsilon}}{1 - (n-1){\rm e}^{-\epsilon}},
\end{equation}

where $f$ is the kernel density associated with the Laplacian kernel $\z_X$.

\end{theorem}

In particular, consider the metric space family $tX$ for some finite subset $X$ of $\R^d$, the weighting vectors $w_t$, and kernel density estimators $f_t$ associated to $\z_{tX}$. As $t \rightarrow \infty$, the RHS of Eqn. \ref{eqn:w_kd_error_bound} converges to zero. Theorem \ref{thm:w_kd_error_bound} shows that for sufficiently large $t$, the solution to the one class SVM induced by $\z_{tX}$ is approximated with concrete bounds by $1/nf$. Na{\"i}vely, $w$ can be obtained by finding the inverse $\z_{tX}^{-1}$. Assuming $\z_{tX}$ has been calculated, this can take $O(n^{\omega})$, where currently $2 \leq \omega \leq 3$; however $1/nf$ can be calculated in $O(n^2)$, with much smaller constant multipliers in practice.

For sufficiently nice kernels, Theorem 1.1 and Corollary 1.1 in \cite{kd_estimation_with_nn} give a way to estimate the kernel density function $f$ in an asyptotically unbiased way via a normalized count of points falling in an appropriately small rectangular region, assuming the points are sampled from the probability distribution defined by the kernel. This provides a further estimate of the kernel density $f$ by

\begin{equation}
    \tilde{f}(x) = \frac{\abs{R_h(x) \cap X}}{n(2h)^{d}}
\end{equation}

where $X \subset \R^d$ a dataset with $\abs{X} = n$, and $R_h(x)$ is a $d$-cube of side length $2h$ centered at $x \in \R^d$. $\abs{R_h(x) \cap X}$ denotes the number of points in $X$ that lie in the set $R_h(x)$. This gives us a final approximation of $w$ by $\frac{1}{n\tilde{f}(x)} = \frac{(2h)^d}{\abs{R_h(x) \cap X}}$. Employing a k-d tree structure for efficient nearest neighbor search, the quantity $\abs{R_h(x) \cap X}$ can be calculated on average in $O(log(n))$ time, after an average build time cost of $O(nlog(n))$. An empirical demonstration of these approximations is given in the supplementary material. It should be noted that while the bound on the approximation in Thm. \ref{thm:w_kd_error_bound} is guaranteed when the space is scattered--that is, when the scale parameter $t$ is chosen to be large enough--in practice, we have seen that $t$ can be chosen to be much smaller, with reasonable approximation error. The data set depicted in the supplementary material requires a $t$ value of approximately $5.8 \times 10^4$ to be scattered, whereas $t=50$ was used. In practice, it was seen that the approximation $\frac{1}{\abs{R_h(x) \cap X}}$ has lower error than the normalized version with numerator $(2h)^d$.

\section{Outlier detection} \label{sec:outlier_detection}



\begin{table}[ht]
    \centering
    \begin{tabular}{lcccccc}
    \toprule
         Name & Feat  & Prec@10 & Rec@10 & F1@10 & AUC  &  AUC RDOS~\cite{tang2017local}\\
         \midrule
         breastw    & 9    & 0.97 & 0.97 & 0.97 & 0.99\\
         cardio     & 21   & 0.94 & 0.66 & 0.77 & 0.98\\
         glass      & 9    & 0.33 & 0.25 & 0.19 & 0.71  & {\bf 0.89}\\        httpKDD    & 3    & 0.94 & 1.00 & 0.97 & {\bf 0.99}  & 0.97\\
         ionosphere & 33   & 0.93 & 0.85 & 0.89 & {\bf 0.95}  & 0.94\\
         lympho     & 18   & 0.38 & 1.00 & 0.55 & 0.98 & \bf{1.00}\\
         pendigits  & 16   & 0.75 & 0.94 & 0.83 & {\bf 0.99} & 0.97\\
         pima       & 8    & 0.85 & 0.13 & 0.22 & 0.66       & {\bf 0.73}\\
         shuttle    & 9    & 0.97 & 1.00 & 0.98 & {\bf 1.00} & 0.98\\
         vowels     & 12   & 0.80 & 0.71 & 0.75 & 0.97\\
         wbc        & 30   & 0.73 & 1.00 & 0.84 & {\bf 0.98} & {\bf 0.98}\\
         wine       & 13   & 0.60 & 1.00 & 0.75 & 0.92\\
         \bottomrule
    \end{tabular}
    \caption{Performance metrics of our anomaly detection algorithm applied to the benchmark data sets of~\cite{Rayana:2016}. We report the best AUC that was achieved by RDOS~\cite{tang2017local}. In each case, optimal RDOS performance exceeded the performance of prior work.}
    \label{tab:outlier}
\end{table}
It is possible to apply weighting vectors to anomaly detection. Our algorithm, detailed below, is either competitive with or better than state-of-the-art techniques at outlier detection on benchmark data sets. The broad idea is as follows. Recall that for a given set of $m$ points $X\subset\R^n$, the weighting vector assigns a real-valued scalar value to each $x\in X$. We call this value the {\em weighting score} of $x$. The weighting score of $x$ is typically large when $x$ is a large distance from all other points of $X$. By definition, outliers of $X$ differ significantly from most other observations, and as a result, it is natural to expect the weighting score of such points to be large.

The foregoing can be used as the foundation of a simple outlier detection algorithm. Given a set of inlier data $Y\subset\R^{n}$:
\begin{enumerate}
\item Normalize $Y$ to have mean 0 and unit variance in each feature. Denote the normalizer operation by $\Phi$.\label{item:step1}
\item For any point $x \in\R^{n}$ with $x \not\in Y$, compute the weighting vector $w$ of $\Phi(Y\cup \set{x})$.\label{item:step}
\item If the weighting score of $\Phi(x)$ is among the $k$ largest entries in $w$, then classify $x$ as an outlier.
\end{enumerate}
We apply this algorithm to 12 data sets that have been reported in prior work, and we compare our results to those of~\cite{tang2017local} in Table~\ref{tab:outlier}.  The data sets reported here have inliers and outliers  labeled.  The weighting vector $w$ depends on a parameter, $t$, which must be chosen according to some objective metric. To find an optimal choice of $t$, we perform a search.  First, we split the inlier data into training, validation and testing sets, and then we randomly distribute the labeled outliers into the validation and testing sets with equal probability.  The training set serves the role of the set $Y$ in step~\ref{item:step1} in the above algorithm. Using the training and validation sets, we perform a parameter search over values of $t\in\set{1\times 10^{j},5\times 10^{j}:-5\leq j\leq 1}$.  The model identifies $x$ as an outlier if its weighting score is among the $k$ largest values.  For several metrics reported in Table~\ref{tab:outlier}, we fixed the threshold $k=10$ and indicated this with an ``@10'' suffix. This is a somewhat arbitrary choice, but it seems to work well in practice.

The optimal choice of $t$ is found by selecting the value of $t$ that maximizes the area under the receiver operating characteristic curve (AUC), which summarizes the results of varying the rank-order threshold $k$ that is used to identify outliers.  We select the value of $t$ with largest AUC, and then report results for $k=10$ applied to the testing set.   Computing the weighting vector in step~\ref{item:step} for the space $\Phi(Y\cup\set{x})$ can be done efficiently, using the standard block matrix inversion formula for $2\times 2$ matrices, and pre-computing the matrix inverse $Z_{\Phi(Y)}^{-1}$ one time. The computational cost of this is one matrix inverse operation, followed by a matrix-vector multiplication for each element in the testing set.  In practice, a sample of at most 1000 inliers effectively serve as a training set.



\section{Conclusion and future work} \label{sec:conclusions}
We have introduced the concept of magnitude and weighting vector into the machine learning community, and present theoretical as well as empirical evidence that the weighting vector can be effectively used as a boundary detector of a densely, uniformly sampled data set in Euclidean space. We detail how the weighting vector can be seen as the solution to a generalized SVM, and give methods to efficiently approximate it via nearest neighbor methods. We define an anomaly detection method based on the weighting vector, and show that it is competitive with state of the art. 

The weighting vector has solid, although abstract, topological and geometric foundations, and sits in connection to SVMs, nearest neighbor methods, and machine learning. In future work, we intend to explore further the connection of SVMs to weighting and magnitude, as well as the connection with nearest neighbor methods. In particular, the approximation given in \ref{thm:w_kd_error_bound} seems to empirically hold for $t$ smaller than the range specified in the theorem. We also have preliminary, ongoing work on a neural network layer inspired by the weighting vector. Finally, the potential for applicability to graphs is touched on briefly in the supplementary materials, and is of interest for further investigation.

\begin{ack}
The authors would like thank Mark Meckes for reading through early manuscripts of this paper and providing extremely useful feedback and discussions.
\end{ack}

\bibliographystyle{abbrv}
\bibliography{bibliography}

\appendix

\section{Appendix}\label{sec:appendix}

\subsection{Active learning Experiments}
\begin{figure*}[ht]
{\small
\begin{tabular}{ccc}
  \includegraphics[scale=0.27, bb=0 0 432 288]{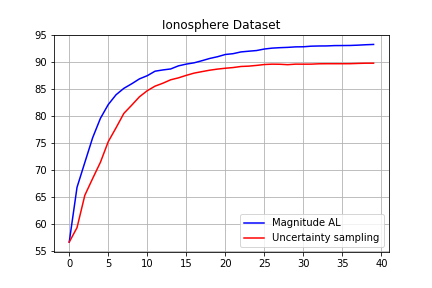} 
  &   
  \includegraphics[scale=0.27, bb=0 0 432 288]{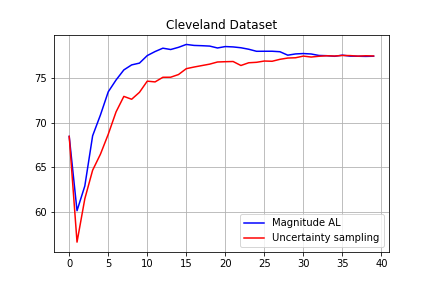} 
  &
  \includegraphics[scale=0.27, bb=0 0 432 288]{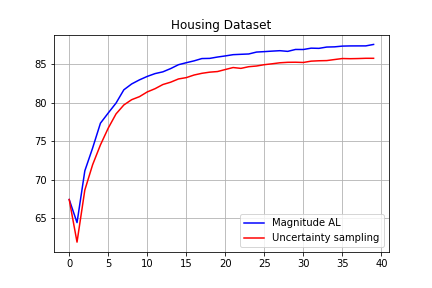} \\
 \includegraphics[scale=0.27, bb=0 0 432 288]{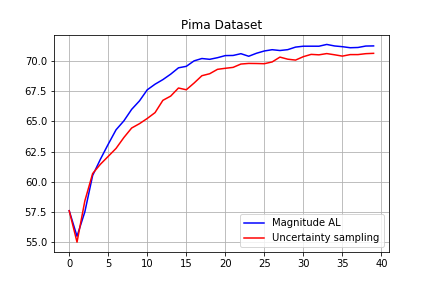} 
 &
 \includegraphics[scale=0.27, bb=0 0 432 288]{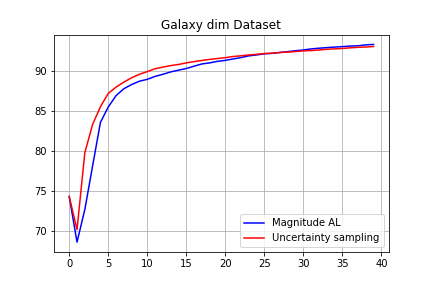} 
 &
 \includegraphics[scale=0.27, bb=0 0 432 288]{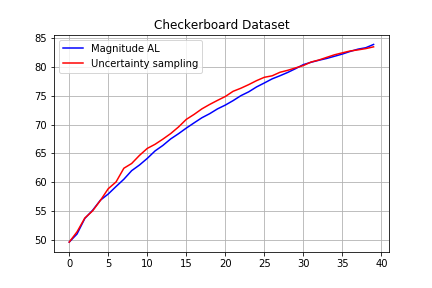} \\
\end{tabular}
\caption{Active learning results comparing the weighting vector query strategy vs the uncertainty sampling strategy. Average over 100 runs.}
\label{fig:al_results}
}
\end{figure*}

Let $\mathcal{L}$ (the labeled dataset) and $\mathcal{U}$ the (unlabeled dataset) be two subsets of the available pool of training data $X$, with $X=\mathcal{U} \cup \mathcal{L}$ and $\mathcal{U} \cap \mathcal{L}=\emptyset$. An iteration of the algorithm will pick some points in  $\mathcal{U}$ to be labeled by an oracle (transferring them to  $\mathcal{L}$). The current model will be then updated using the new updated dataset  $\mathcal{L}$ and its corresponding labels. For simplicity we will state the algorithm for a binary classification problem i.e. when $L=\{L_{0},L_{1}\}$, however it can be trivially extended to a multi-class problem. 

The intuition behind the algorithm is simple: at each iteration $i$, we assign every training data point to one of the sets  $\tilde{X_0}$ or $\tilde{X_1}$ according to its predicted label by the current classifier $f_i$. We will calculate the corresponding weight vectors $w_{\tilde{X_{0}}}$ and $w_{\tilde{X_{1}}}$. Then, we choose to label the point with the minimum value (interior point) and the with the maximum value (likely to be in the boundary) for both sets $\tilde{X_0}$ and $\tilde{X_1}$. By choosing  this way we are aiming to: (a) reinforce, validate and refine high confidence classifier information (labels) acquired in prior iterations (exploitation) and (b) to acquire labels in the predicted class boundaries where our classifier confidence is potentially lower (exploration). The proposed active learning algorithm is stated in Algorithm \ref{alg:al}. 

\begin{algorithm}
    \caption{Active learning via weighting vector.}
    \label{alg:al}
    \begin{algorithmic}
        \Require{ Data set $X$ }
        \State{ $\mathcal{L}=\emptyset$; $\mathcal{U}=X$ }
        \State{ initialize $\mathcal{L}$ ;  $\mathcal{U}=X-\mathcal{L}$;  with it's corresponding $\mathcal{Y_{\mathcal{L}}}$ }
        \State{ $f={\bf train\_classifier}(\mathcal{L}$,$\mathcal{Y_{\mathcal{L}}}$) }
        \While{ (not converged) {\bf or} (labeling budget not reached) }
            \State{ $\tilde{X_i} = \{ x \in X \mid f(x)  = i \}$ for $i = 0, 1$. }
            \State{ calculate weighting vectors $w_{\tilde{X_{i}}}$ }
            \State{ $Q_{\min,i}=\underset{\mathcal{U}}{\arg\min}  \ {\text{abs} (w_{\tilde{X_{i}}}}) $ for $i = 0, 1$ }
            \State{ $Q_{\max,i}=\underset{\mathcal{U}}{\arg\max}  \ {\text{abs}(w_{\tilde{X_{i}}}} )$ for $i = 0, 1$ }
            \State{ $\mathcal{Y_{\mathcal{Q}}}$=query\_labels($Q_{\min,0}$,$Q_{\max,0}$,$Q_{\min,1}$,$Q_{\max,1}$) }
            \State{ $\mathcal{L}=\mathcal{L} \cup \{Q_{\min,0}$,$Q_{\max,0}$,$Q_{\min,1}$,$Q_{\max,1}\}$ }
            \State{ $\mathcal{Y_{\mathcal{L}}}=\mathcal{Y_{\mathcal{L}}} \cup  \mathcal{Y_{\mathcal{Q}}}$ }
            \State{ $\mathcal{U}=X-\mathcal{L}$; }
            \State{ $f={\bf train\_classifier}(\mathcal{L},\mathcal{Y_{\mathcal{L}}}$) }
        \EndWhile
    \Ensure{$f$}
    \end{algorithmic}
\end{algorithm}

In order to assess the effectiveness of the weighting-vector-based active learning (AL) algorithm proposed, we compared Algorithm \ref{alg:al} to the simplest but highly effective and most commonly used query AL framework: uncertainty sampling \cite{lewis94}. In this framework, the AL algorithm queries the instances for which it is least certain about how to label (i.e. for many algorithms $p(label\|x) \approx 0.5$ or where the decision function is close to $0$). For simplicity we used a kernelized Ridge regression model \cite{cristianini2000} (also refer as to LS-SVM \cite{lssvm} or proximal SVM \cite{ProximalSVM}). Laplacian kernels were used both as magnitude to calculate the weighting vector and as classification kernel ($k(x, y) = \exp( -\gamma \| x-y \|_1)$ with $\gamma=0.1$. At each iteration of Algorithm \ref{alg:al} the classifier learned after obtained labels from the oracle has the form $f(x)=K(x,\mathcal{L})'w-w_0$, where $w_0$ is the bias term.

We performed experiments on five classic benchmark datasets from the UCI repository taking 67\% of the data as training pool and the remaining 33\% as a testing set. Note that the weighing-vector-inspired algorithm chooses $4$ points per iterations so we picked the four more uncertain points for the uncertainty sampling algorithm to be fair. 

Figure \ref{fig:al_results} shows average performance curves over 100 runs. The performance from the weighting vector algorithm seems to perform better in four out of the five datasets and slightly worse on the Galaxy dim. and Checkerboard datasets.

\section{Useful properties of magnitude}
\label{sec:properties}
In this section, we offer some techniques that are useful when working with weighting vectors. We discuss how the computation of the weighting vector may be effectively computed by breaking the computation into smaller pieces and ``gluing'' the results together.

\subsection{Inclusion-Exclusion for Weight and Magnitude}
We demonstrate a practical way to calculate the weighting  vector for a set $Z \coloneqq X \cup Y$ that is the union of two finite $X,Y\subset \R^n$. To approach this, first we investigate the case when $X$ and $Y$ are disjoint. Then we will look at the case $Y \subset X$, and show how to calculate either $w_X$ or $w_Y$ when one knows the other. Finally we will state an inclusion-exclusion principle for magnitude, as  well  as the weighting vector.

Before proceeding, we recall the definition of the \textit{Schur complement}.

\begin{definition} Let
$M \coloneqq 
\begin{bmatrix}
A & B \\
C & D
\end{bmatrix}$
be the block matrix where the matrices $A, B, C, D$ are of dimensions $n\times n,  n \times  m, m \times n,$ and $m \times m$ respectively. If $D$ is invertible, then the \textit{Schur complement} of $D$ in $M$ is the $n \times n$ matrix

    \begin{equation*}
        M/D = A - BD^{-1}C.
    \end{equation*}
    
\noindent Similarly, if $A$ is invertible, then the Schur complement of $A$ in $M$ is the $m \times m$ matrix

    \begin{equation*}
        M/A = D - CA^{-1}B.
    \end{equation*}
\end{definition}

Let $\emptyset \neq Y \subset X \subset \R^{n}$ be finite sets. Without loss of generality, we can index the points of $X$ such that the first $\abs{Y}$ of them correspond to those points in $Y$. Then we can see that $\z_X$ can be written as a block matrix
\begin{align}\label{eqn:zeta_block}
\z_X = 
\begin{bmatrix}
\z_Y & \z_{Y, \bar{Y}} \\
\z_{Y, \bar{Y}}^T & \z_{\bar{Y}}
\end{bmatrix},
\end{align}
where $\bar{Y} = X \setminus Y$, and $\z_{Y, \bar{Y}}$ denotes the submatrix of $\z_X$ formed by taking the rows corresponding to $Y$ and columns corresponding to $\bar{Y}$. We can now rewrite the formula $\z_Xw = \one$ using equation \ref{eqn:zeta_block} as the system of equations
\begin{align*}
\z_Y \restr{w_X}{Y} + \z_{Y, \bar{Y}} \restr{w_X}{\bar{Y}} 
&=
\one_Y \\
\z_{Y, \bar{Y}}^T \restr{w_X}{Y} + \z_{\bar{Y}} \restr{w_X}{\bar{Y}} 
&=
\one_{\bar{Y}},
\end{align*}
where $\one_Y$ and $\one_{\bar{Y}}$ are respectively the $\abs{Y} \times 1$ and $\abs{\bar{Y}} \times 1$ column vectors of all ones. Since both $\z_Y$ and $\z_{\bar{Y}}$ are invertible, we can form both of the Schur complements $\z_X/\z_Y$ and $\z_X / \z_{\bar{Y}}$. With these in hand, we can write
\begin{align}
\restr{w_X}{Y}
&=
(\z_X/\z_{\bar{Y}})^{-1} (\one_Y - \z_{Y, \bar{Y}} w_{\bar{Y}}) \label{eqn:disjoint_gluing_Y} \\
\restr{w_X}{\bar{Y}}
&=
(\z_X/\z_{Y})^{-1} (\one_{\bar{Y}} - \z_{Y, \bar{Y}}^T w_{Y}), \label{eqn:disjoint_gluing_Ybar} 
\end{align}
where $w_Y$ and $w_{\bar{Y}}$ are the weight vectors for $Y$ and $\bar{Y}$ respectively, and $\restr{w_X}{Y}$ is the weight vector of $X$, restricted to those indices corresponding to $Y$. Thus if we know $w_Y$ and $w_{\bar{Y}}$, equations \ref{eqn:disjoint_gluing_Y} and \ref{eqn:disjoint_gluing_Ybar} give a way to compute $w_X$.

Next, for finite sets $Y \subset X \subset \R^n$ we wish to calculate either the weight vector $w_X$ or $w_Y$ given the other. 

\begin{definition}\label{def:rho}
For a block matrix $M \coloneqq \begin{bmatrix} 
A & B \\
C & D 
\end{bmatrix}$ with $A$ invertible, define 
\begin{equation*}
    \rho_{MA} \coloneqq \begin{bmatrix}
A^{-1} B (M/A)^{-1} C A^{-1} & -A^{-1} B (M/A)^{-1} \\
-(M/A)^{-1} C A^{-1} & (M/A)^{-1}
\end{bmatrix}.
\end{equation*}
\end{definition}

\noindent Recall that for a block matrix $M$ as in Definition \ref{def:rho}, 
\begin{equation}\label{eqn:M_A_rho}
    M^{-1}
    =
    \begin{bmatrix}
    A^{-1} & 0 \\
    0 &0 \\
    \end{bmatrix}
    + \rho_{MA}.
\end{equation}












\begin{definition}
For $Y \subseteq X \subset \R^n$ finite sets, assume $\z_X$ is in block matrix format as in Equation \ref{eqn:zeta_block}. Define the matrix
\begin{equation*}
\rho_{XY} = \rho_{\z_X \z_Y}
\end{equation*}
\noindent where $\rho_{XY}$ is taken to be the zero matrix when $Y  = X$, and $\rho_{XY}$ is taken to be $\z_X$ when $Y = \emptyset$.
\end{definition}

\begin{lemma}\label{lem:weight_vector_subset}
For finite sets $Y \subset X \subset \R^n$, let $P_{XY}$ be a permutation matrix such that

\begin{equation*}
    P_{XY} \z_X P_{XY} = \begin{bmatrix}
    \z_Y & \z_{Y \bar{Y}} \\
    \z_{Y \bar{Y}}^T & \z_{\bar{Y}}.
    \end{bmatrix}
\end{equation*}
\noindent Then
\begin{align*}
w_X
&=
P_{XY}\begin{bmatrix}
w_Y \\
0
\end{bmatrix}
 + 
 P_{XY}\rho_{XY}\one,
\text{ \hspace{3pt} and} \\
 \Mag{X} 
 &= 
 \Mag{Y} 
 + 
 \one^T\rho_{XY}\one.
\end{align*}
\end{lemma}

\begin{proof}
Set $M = P_{XY}\z_{X}P_{XY}$, employ Eqn. \ref{eqn:M_A_rho}, and multiply on the right by $\one$.
\end{proof}



\noindent We can now calculate the weight vector of $X \cup Y$ where $X$ and $Y$ are not necessarily disjoint. This can be viewed as an inclusion-exclusion principle that applies to weight vectors as well as magnitude.

\begin{theorem}
For finite sets $X, Y \subset \R^n$, set $Z = X \cup Y$. Then we have

\begin{align*}
    w_{Z} 
    &=
    P_{Z X}\left( \begin{bmatrix}
    w_X \\ 0
    \end{bmatrix} 
    + 
    \rho_{Z X}\one \right) 
    +
    P_{Z Y} \left( \begin{bmatrix}
    w_Y \\ 0
    \end{bmatrix} 
    + 
    \rho_{Z Y} \one \right) \\
    &- 
    P_{Z X \cap Y} \left( \begin{bmatrix}
    w_{X \cap Y} \\ 0
    \end{bmatrix} 
    - 
    \rho_{Z X \cap Y}\one \right), \text{\hspace{3pt} and} \\
    \Mag{Z}
    &=
    \Mag{X}
    + 
    \Mag{Y}
    -
    \Mag{X \cap Y} \\
    &+
    \one^T \rho_{Z X} \one  
    +
    \one^T \rho_{Z Y} \one 
    -
    \one^T \rho_{Z X \cap Y} \one.
\end{align*}

\end{theorem}

\begin{proof}
This follows by applying Lemma \ref{lem:weight_vector_subset} to each subset considered, e.g.

\begin{align*}
    w_Z 
    &= 
    P_{ZX}\begin{bmatrix}
    w_X \\ 0
    \end{bmatrix}
    +
    P_{ZX}\rho_{ZX}\one.
\end{align*}
\end{proof}



    

\subsection{Proof of theorem \ref{thm:w_kd_error_bound}}

\begin{figure}[ht]
    \centering
    \includegraphics[width=0.9\textwidth, bb=0 0 1339 485]{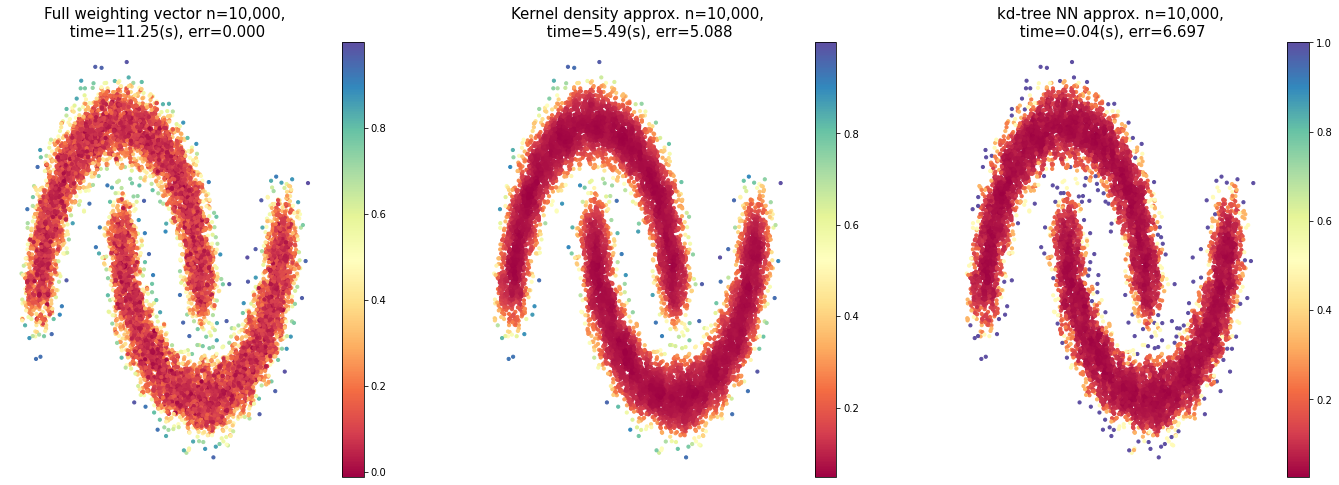}
    \caption{Each image depicts the weighting vector or an approximation computed for the moons \cite{scikit-learn} data set consisting of 10,000 points. Error computed is the $l_2$ norm of the difference between the approximating vector and $w$. The left image shows the full weighting vector $w$ with $t=50$. The center image shows the approximation by $1/nf$ with again $t=50$. The right image shows approximation by $1/n\tilde{f}$ with ball search radius $0.03$ using the $l_{\infty}$ norm.}
    \label{fig:weight_approx}
\end{figure}

\begin{theorem*} 
For $X$ scattered, and $x \in X$,

\begin{equation} 
    \abs{w(x) - \frac{1}{nf(x)}} \leq \frac{n(n-1)^2{\rm e}^{-2\epsilon} + n(n-1){\rm e}^{-\epsilon}}{1 - (n-1){\rm e}^{-\epsilon}},
\end{equation}

where $f$ is the kernel density associated with the Laplacian kernel $\z_X$.

\end{theorem*}

\begin{proof}
For $a, b \in X$, write $\mu(a, b) := \z^{-1}_X(a, b)$. Let $\epsilon$ denote the smallest distance between distinct points of $X$. Proposition 2.1.3 in \cite{Leinster2013TheMO} gives that for $X$ scattered we can write

    \begin{equation*}
        \mu(a, b) = \sum_{k = 0}^{\infty} (-1)^k \mu_k(a, b)
    \end{equation*}
    
where 

    \begin{equation*}
        \mu_k(a, b) = \sum_{a = a_0 \neq \cdots \neq a_k = b} \z_X(a, a_1) \z_X(a_1, a_2) \cdots \z_X(a_{k - 1}, b)
    \end{equation*}
    
where the sum is over all $a_0, \dots, a_k \in X$ with $a_0 = a$, $a_k = b$, and $a_{j - 1} \neq a_j$ for all $1 \leq j \leq k$. In particular, $\mu_0$ is the identity matrix. The proof of Prop. 2.1.3 in \cite{Leinster2013TheMO} gives that $\mu_k(a, b) \leq ((n - 1){\rm e}^{-\epsilon})^k$ for all $a, b \in X$. This gives that 

    \begin{equation*}
        \mu(a, b) = \sum_{k = 0}^{\infty}(-1)^k \mu_k(a, b) \leq \sum_{k = 0}^{\infty} \mu_k(a, b) \leq \sum_{k = 0}^{\infty} ((n - 1){\rm e}^{-\epsilon})^k = \frac{1}{1 - (n - 1){\rm e}^{-\epsilon}}.
    \end{equation*}
    
The final equality due to employing geometric series, since $(n - 1){\rm e}^{-\epsilon} < 1$ by assumption of $X$ being scattered. Let $w$ be the weighting vector of $X$, and note that for $a \in X$

    \begin{align}
        w(a) \nonumber
        &=
        \sum_b \sum_{k = 0}^{\infty} (-1)^k\mu_k(a, b)
        =
        \sum_b \mu_0(a, b) + \sum_b \sum_{k = 1}^{\infty} (-1)^k \mu_k(a, b) \nonumber \\ 
        & =
        1 + \sum_b \sum_{k = 0}^{\infty} (-1)^k \mu_k(a, b). 
    \end{align}
    
Then 

    \begin{align}
        & \abs{w(a) - \frac{1}{nf(a)}} \nonumber
        = 
        \abs{\frac{1}{nf(a)}(nw(a)f(a) - 1)} 
        \leq 
        \abs{\frac{1}{nf(a)}}\abs{nw(a)f(a) - 1} \\[10pt]
        & \leq 
        \abs{nw(a)f(a) - 1}
        =
        \abs{nw(a)\frac{1}{n}\sum_b \z_X(a, b) - 1}
        =
        \abs{w(a) + w(a)\sum_{b \neq a}\z_X(a, b) - 1} \nonumber \\[10pt]
        & \leq
        \abs{w(a) - 1} + \abs{w(a)\sum_{b \neq a} \z_X(a, b)}
        \leq \abs{1 + \sum_b \sum_{k = 0}^{\infty} (-1)^k \mu_k(a, b) - 1} + \abs{\frac{n(n - 1){\rm e}^{-\epsilon}}{1 - (n - 1){\rm e}^{-\epsilon}}} \nonumber \\[10pt]
        & \leq
        \abs{\sum_b \frac{((n - 1){\rm e}^{-\epsilon})^2}{1 - (n - 1){\rm e}^{-\epsilon}}} + \frac{n(n - 1){\rm e}^{-\epsilon}}{1 - (n - 1){\rm e}^{-\epsilon}}
        =
        \frac{n(n - 1)^2{\rm e}^{-2\epsilon} + n(n - 1){\rm e}^{-\epsilon}}{1 - (n - 1){\rm e}^{-\epsilon}}.
    \end{align}
\end{proof}

For a positive definite kernel $K$ on $X$ having $K(x, x) = 1$ for all $x \in X$, denote by $K_{max}$ the largest value that $K$ takes on distinct points of $X$. Say $X$ is \emph{$K$-scattered} if $K_{max} < \frac{1}{n - 1}$. Then Thm. \ref{thm:w_kd_error_bound} holds for $K$ with ${\rm e}^{-\epsilon}$ replaced by $K_{max}$. Theorem \ref{thm:w_kd_error_bound} heavily leverages Prop. 2.1.3 in \cite{Leinster2013TheMO}, which in fact can be modified slightly to apply to any real valued matrix whose diagonal entries are 1, and whose non-diagonal entries have absolute value strictly less than $\frac{1}{n - 1}$, giving a nice formula for the entries of the inverse matrix.

Figure \ref{fig:weight_approx} shows the complete weighting vector and two approximation methods outlined in Section \ref{sec:nn_approx} for the moons data set \cite{scikit-learn}.

\subsection{Weighting vector on graphs}

When computing the weighting vector of a graph, one typically constructs a metric space whose points are indexed by the nodes of the graph, and distance derived from the structure of the edges; e.g. by taking the shortest path length between two nodes. We have observed empirically that the resistance distance \cite{resistance-dist} seems to behave better than the distance given by shortest path length when computing the weighting vector.

We can visualize the inverse relationship between $w(\nu)$ for $\nu$ some node in the graph, and the number of nearest neighbors to $\nu$ in the corresponding metric space. In Figure \ref{fig:graph_nn_weight_corr}, the order of each vertex is plotted with respect to its magnitude. Although order is not exactly proportional to number of nearest neighbors in the metric space, under most reasonable metrics and with a large value of $t$, it serves as a very good proxy.

Given suitably large $t$, the graph is scattered--namely, we need $t > \frac{\log(n-1)}{\epsilon}$, where $n$ is the number of nodes in the graph, and $\epsilon$ is the minimum distance between any two nodes. As seen in Figures 6 and 7, the inverse relationship $w(x) \approx \frac{1}{nf(x)}$ is clear, as predicted by Theorem 7. 

\begin{figure}
\centering
\includegraphics[width = 0.6\textwidth, bb=0 0 450 250]{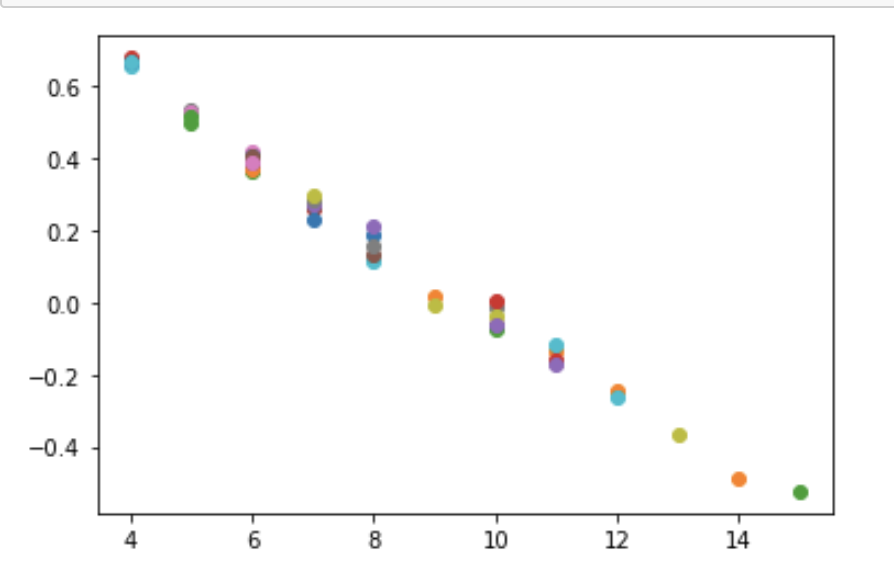}
\caption{This image depicts number of neighbors (x-axis) vs weighting vector value (y-axis) for each node in the graph. The graph here is an Erdos-Renyi graph \cite{erdos59a} with 50 nodes, and 15\% connected. The distance metric is resistance distance \cite{resistance-dist}, and we have $t=6$, $\epsilon = 0.182$. It's generated using the NetworkX \cite{SciPyProceedings_11} package.}
\label{fig:graph_nn_weight_corr}
\end{figure}

\begin{figure}
\centering
\includegraphics[width = 0.6\textwidth, bb=0 0 450 250]{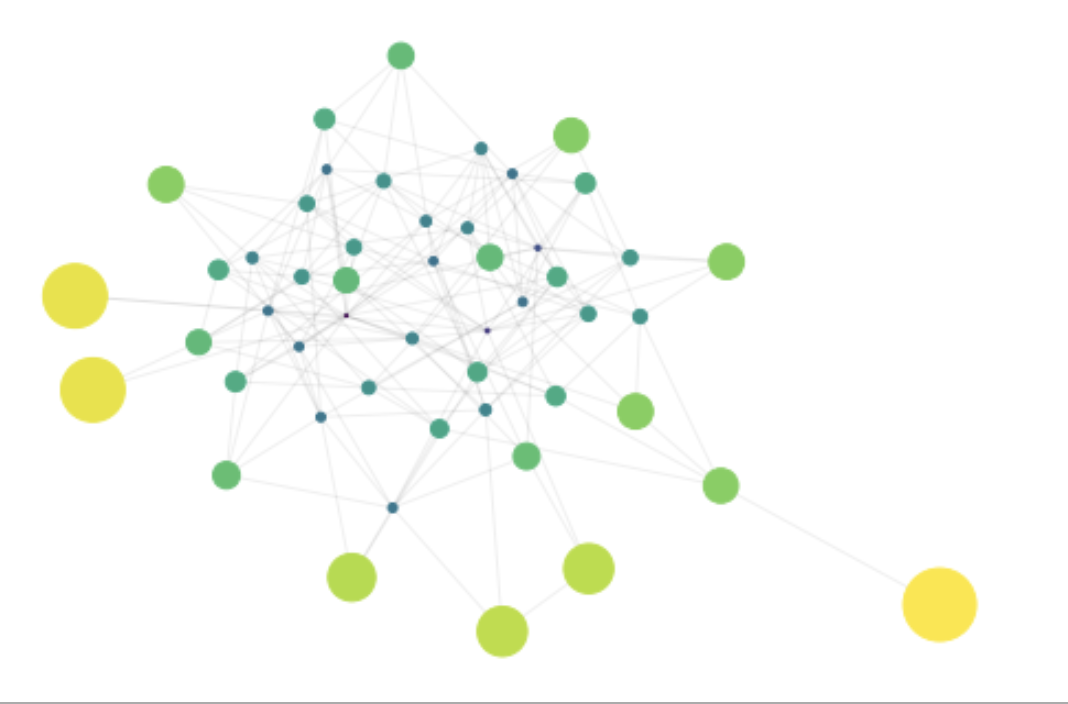}
\caption{A visualization of the same Erdos-Renyi graph \cite{erdos59a}, where the size of node $x$ is  defined by $e^{(4.1)(w(x) - \min(w))}$. The inverse relationship between magnitude and number of neighbors can clearly be seen.}

\end{figure}

\end{document}